\documentclass[runningheads]{llncs}
\usepackage[T1]{fontenc}
\newcommand{\ignore}[1]{}
\usepackage{graphicx}
\usepackage{booktabs}
\usepackage{multirow}
\usepackage{siunitx}
\usepackage{adjustbox}
\usepackage[table]{xcolor}
\usepackage{graphicx} 
\usepackage{tikz}
\usepackage{pgfplots}
\usepackage{amsmath,amssymb}
\usepackage{bm}
\usepackage{pdflscape}
\usetikzlibrary{calc,math,chains,decorations.pathreplacing}
\usepackage{algorithm}
\usepackage{algorithmic}
\usepackage{xspace} 
\newcommand{\dominates}{\prec}
\newcommand{\gsemo}{GSEMO\xspace}

\usepackage{cite}

\begin{document}
%

\title{Evolutionary Algorithms and Multi-Objective
Minimum Spanning Trees with Limited Distinct
  Weight Values}

 \titlerunning{EAs and MO-MSTs with Limited Distinct
  Weight Values}
%
\author{Narges Tavassoli Kejani\inst{1}\orcidID{0009-0003-7529-2298} \and
Andrew M. Sutton\inst{2}\orcidID{0000-0003-1295-6715} \and
Frank Neumann\inst{3}\orcidID{0000-0002-2721-3618}}

%
%
\institute{Univ. Lille, CNRS, Centrale Lille, UMR 9189 CRIStAL, F-59000, France \\
\and
University of Minnesota Duluth, Minnesota, USA\\
\and Optimisation and Logistics,
School of Computer Science and Information Technology, 
Adelaide University, Adelaide, Australia\\
}
%

\maketitle              
\begin{abstract}
Evolutionary algorithms have been used for a wide range of
multi-objective combinatorial optimization problems. Despite practical
success, theoretical results on the runtime of evolutionary algorithms
for multi-objective combinatorial problems are rather limited. One
classical problem that has been investigated is the multi-objective
minimum spanning tree problem for which runtime bounds have been
obtained to compute all extremal corner points of the Pareto front.
With this paper, we provide some more detailed insights into the
structure of the Pareto front when the edge weights take on a small
number of distinct values. Based on these insights, we derive new runtime results for evolutionary multi-objective algorithms and complement our theoretical results with experimental investigations.

\end{abstract}

\section{Introduction}
Evolutionary algorithms have been widely applied in the context of multi-objective optimization~\cite{DBLP:books/daglib/0004745} since the population of an evolutionary algorithm can be evolved into a set of trade-off solutions with respect to the given objective functions. 
A wide range of evolutionary multi-objective algorithms have been designed and shown to be highly successful over the last thirty years. This includes solving classical multi-objective combinatorial problems such as the multi-objective minimum spanning tree problem for which different evolutionary approaches have been developed~\cite{DBLP:journals/eor/ZhouG99,DBLP:conf/cec/KnowlesC01,DBLP:journals/ec/BossekG24}.

Despite these many successful applications, understanding from a
theoretical perspective of which type of problem can provably be solved by evolutionary multi-objective algorithms is still a challenging task. Different results on various aspects of evolutionary multi-objective optimization have been obtained using rigorous runtime analysis~\cite{BookDoeNeu}, which is a major direction in the area of theory of evolutionary computation.

While there are many results on artificial benchmarks problems that
showcase the working principles of evolutionary algorithms (see, e.g.~\cite{Giel_2003,DBLP:journals/tec/LaumannsTZ04,DBLP:journals/ai/ZhengD23,DBLP:journals/ai/Opris26}), as well as results for constrained single-objective optimization problems~(see, e.g.~\cite{DBLP:journals/nc/NeumannW06,DBLP:journals/algorithmica/KratschN13,DBLP:journals/ec/FriedrichHHNW10,QianLFTTCS23}) that are solved via multi-objectivization~\cite{DBLP:conf/emo/KnowlesWC01}, results for classical combinatorial multi-objective optimization problems are very limited.
We are only aware of studies regarding the runtime behavior of MOEAs for multi-objective shortest paths~\cite{DBLP:journals/ec/Horoba10} and multi-objective minimum spanning trees~\cite{Neumann_2007}.
The first results of an algorithm called \gsemo for multi-objective minimum spanning trees appeared at PPSN 2004~\cite{DBLP:conf/ppsn/Neumann04} and its journal extension has been published in \cite{Neumann_2007}.
Current available results for MOEAs and multi-objective minimum spanning trees consider the computation of extremal corner points of the Pareto front and provide runtime bounds for algorithms such as GSEMO~\cite{Neumann_2007}, NSGA-II~\cite{DBLP:conf/ijcai/CerfDHKW23} and MOEA/D~\cite{DBLP:conf/nips/DoN0S23}. The expected time for MOEA/D to compute these extremal corner points is polynomial as shown in \cite{DBLP:conf/nips/DoN0S23}, whereas the runtime bounds shown for GSEMO and NSGA-II in the mentioned articles are only pseudo-polynomial.

In~\cite{Neumann_2007}, Neumann focuses on computing extremal Pareto-optimal objective vectors. The natural question arises whether MOEAs can also compute non-extremal Pareto-optimal objective vectors. To provide new insights into the behavior of MOEAs for the multi-objective minimum spanning tree problem, we consider the case where edge weights take only a small number of distinct values. Specifically, we investigate bi-objective settings where the edge weights are restricted to two or three distinct values and derive structural insights into the geometry of the Pareto front. Based on these insights, we provide runtime analyses of \gsemo showing that the entire Pareto front can be computed in expected polynomial time.

We investigate Pareto front structure and its influence on the
runtime behavior of \gsemo on bi-objective minimum spanning tree
(bi-MST) problems under several different edge weight configurations.
For the two-weight case, we show that all Pareto-optimal objective vectors lie on a single line, enabling \gsemo to compute the entire Pareto front efficiently in expected time $O(m^2 n \log n)$. For the three-weight case, we distinguish between equidistant and non-equidistant settings: in the equidistant case, the efficient boundary of the convex hull consists of a constant number of linear phases, while in the non-equidistant case, some of Pareto-optimal objective vectors do not lie on the convex hull. Finally, for a constructed instance based on $k$ pairs of triangles, we prove that the Pareto front contains exactly $3k+1$ objective vectors, of which exactly $k$ do not lie on the convex hull, and we show that \gsemo computes the entire Pareto front in expected time $O(n^4)$.
Furthermore, we complement our theoretical investigations with experimental ones that study the impact of different instance parameters on the empirical runtime.

The paper is structured as follows. In Section~\ref{sec2}, we introduce the problem and algorithm that is subject to our investigations. 
In Section~\ref{sec3}, we investigate general structural properties of
Pareto-optimal solutions that form the basis for our analysis. We then
apply these insights to analyze multi-objective spanning trees with
two distinct edge weight values in Section~\ref{sec4}  and extend this
analysis to three distinct edge weight values in Sections~\ref{sec5} and~\ref{sec6}. In Section~\ref{sec7}, we discuss how our results translate to popular multi-objective evolutionary algorithms.
We finish with some conclusions. 

\section{Preliminaries}
\label{sec2}
A \emph{bi-Objective Combinatorial Optimization Problem} (BCOP) is commonly defined as follows:
$$\text{Optimize } f(x) = (f_1(x), f_2(x)) \;\text{s.t. }\; x \in \mathcal{D},$$
where objective functions $f_i$ are optimized, $x$ denotes a solution (i.e., a vector of decision variables), and $\mathcal{D}$ represents the discrete set of feasible solutions. The \emph{objective space}, denoted by $\mathcal{Z}$, is defined as the image of $f$. Each solution $x \in \mathcal{D}$ is thus mapped to a point $f(x) \in \mathcal{Z}$.

In a minimization context, a solution $x$ is said to \emph{weakly dominate} another solution $y$, denoted by $x \preceq y$, if and only if for all $i \in \{1,2\}$, $f_i(x) \leq f_i(y)$. Moreover, a solution $x$ is said to \emph{dominate} another solution $y$, denoted by $x \dominates y$, if and only if $x \preceq y$, and there exists at least one index $i \in \{1, 2\}$ such that $f_i(x) < f_i(y)$. 
The dominance relation induces a partial order over the solution space, since some pairs of solutions cannot be compared. Such solutions are called \emph{incomparable}.
A \emph{Pareto-optimal solution} is a solution $x^* \in \mathcal{D}$ for which there exists no solution $x \in \mathcal{D}$ such that $x \dominates x^*$. The set of all Pareto-optimal solutions is called the \emph{Pareto-optimal set}. The image of this set under $f$ is referred to as the \emph{Pareto front} $\mathcal{F}$. The \emph{convex hull} of $\mathcal{F}$, $\mathrm{conv}(\mathcal{F})$, is defined as the smallest convex set containing all Pareto-optimal objective vectors. The \emph{efficient boundary} of $\mathrm{conv}(\mathcal{F})$ is the set of non-dominated points on its boundary. In the bi-objective case, this corresponds to the lower-left boundary of $\mathrm{conv}(\mathcal{F})$.

The \emph{bi-objective Minimum Spanning Tree} (bi-MST) problem is defined as a particular case of BOCP on graphs. It can be formalized as follows:
$$\text{Minimize } f(T) = (f_1(T), f_2(T)) \;\text{s.t. }\; T \in \mathcal{T},$$
where $G = (V, E)$ is a connected graph, $\mathcal{T}$ denotes the set of all spanning trees of $G$, and for $i \in \{1,2\}$ each objective function $f_i : \mathcal{T} \rightarrow \mathbb{R}_{>0}$ assigns a cost to a spanning tree $T \in \mathcal{T}$ according to the $i$-th criterion. Each objective function is typically defined as the sum of edge weights with respect to a given weight function $w_i : E \rightarrow \mathbb{R}_{>0}$, i.e., $f_i(T) = \sum_{e \in T} w_i(e)$. The goal is to identify all Pareto-optimal spanning trees; that is, all trees that are not dominated with respect to the vector-valued objective function $f$. The problem is NP-hard~\cite{ehrgott_2005}.

Let $G = (V, E)$ be a connected graph with $n = |V|$ vertices and
$m = |E|$ edges and associated weight functions
$w_1,w_2 \colon E \to \mathbb{R}_{>0}$. We consider multi-objective
evolutionary algorithms solving the bi-MST problem on $G$.
Following~\cite{Neumann_2007}, let
$w_i^{\min}:= \min_{e \in E} w_i(e)$ and
$w_i^{\max}:= \max_{e \in E} w_i(e)$ denote the minimum and maximum
edge weights with respect to objective $i \in \{1,2\}$, respectively.
Furthermore, define 
\[
w_{\min} := \min_{i \in \{1,2\}} w_i^{\min}, \qquad
w_{\max} := \max_{i \in \{1,2\}} w_i^{\max}, \qquad
w_{ub} := n^2 \cdot w_{\max}.
\]
Solutions are represented as bit strings $s \in \{0,1\}^m$, where $s_j = 1$ if edge $e_j$ is selected and $s_j = 0$ otherwise. 

The fitness function is defined as
$f(s) = (f_1(s), f_2(s))$,
where, for each objective $i \in \{1,2\}$,
\[
f_i(s) := (c(s) - 1)\cdot w_{ub}^2
+ (e(s) - (n - 1))\cdot w_{ub}
+ \sum_{e_j \in E} s_j \, w_i(e_j).
\]
Here, $c(s)$ denotes the number of connected components in the graph represented by $s$, and $e(s)$ denotes the number of edges in this graph. 
The penalty terms ensure feasibility: the first term penalizes disconnected solutions, while the second penalizes solutions containing cycles. 
As a result, only connected graphs with exactly $n - 1$ edges (i.e., spanning trees) can achieve minimal penalty. 
For a spanning tree $T$, let $x_{ij}(T)$ denote the number of edges of type $(i,j)$. 
Moreover, let $x_{w_{\min}}(T)$ and $x_{w_{\max}}(T)$ denote the number of edges in $T$ with weights $(w_{\min}, w_{\min})$ and $(w_{\max}, w_{\max})$, respectively. When the underlying spanning tree is clear from the context, we write $x_{ij}$ instead of $x_{ij}(T)$.

We consider the \gsemo (Global Simple Evolutionary Multi-objective Optimizer) algorithm (see Algorithm~\ref{alg:gsemo}), introduced by Giel~\cite{Giel_2003}. This algorithm can be viewed as a generalization of the well-known $(1+1)$ EA, which has been extensively studied in theoretical analyses, to the multi-objective setting.
The \gsemo maintains a population $P$ of mutually non-dominated solutions. Initially, a solution is generated uniformly at random and inserted into $P$. 
In each iteration, a solution $x \in P$ is selected uniformly at random. A new solution $y$ is created by standard bit mutation, where each bit is flipped independently with probability $\frac{1}{m}$. 
If $y$ is not dominated by any solution in $P$, it is added to the population, and all solutions in $P$ that are dominated by $y$ are removed. Otherwise, $y$ is discarded. 
For our investigations, we consider the expected number of iterations of the while loop until \gsemo has obtained a population $P$ that includes for each Pareto-optimal objective vector a corresponding solution. We call the expected number of iterations of \gsemo as the expected time to obtain a given goal.
\begin{algorithm}[t]
\caption{\gsemo}
\label{alg:gsemo}
\begin{algorithmic}[1]
\STATE \textbf{Input:} Fitness function $f : \{0,1\}^m \rightarrow \mathbb{R}^k$
\STATE Initialize $P \leftarrow \{x\}$, where $x \in \{0,1\}^m$ is chosen uniformly at random
\WHILE{stopping condition not met}
    \STATE Select $x \in P$ uniformly at random
    \STATE Create $y$ by standard bit mutation (each bit flipped independently with probability $\frac{1}{m}$)
    \IF{there is no $z \in P$ such that $f(z) \dominates f(y)$}
        \STATE $P \leftarrow (P  \setminus \{z \in P \mid f(y) \preceq f(z)\}) \cup \{y\}$
    \ENDIF
\ENDWHILE
\end{algorithmic}
\end{algorithm}

We denote by $q_1$ the objective vector of a Pareto-optimal solution that is a minimum spanning tree with respect to the weight function $w_1$. Similarly, we denote by $q_r$ the objective vector of a Pareto-optimal solution that is a minimum spanning tree with respect to the weight function $w_2$.
The following two lemmas were proven by Neumann~\cite{Neumann_2007}.
\begin{lemma}[Lemma~4 in~\cite{Neumann_2007}] \label{lem:neumann}
\gsemo reaches a population consisting only of spanning trees in expected time $O(m \log n)$.
\end{lemma}

The following lemma gives an upper bound on the time to include a solution for each of the objective vectors $q_1$ and $q_r$. It depends on the value of $|P|_{\max}$, which denotes the maximum population size observed during the run of the algorithm.

\begin{lemma}[Lemma~5 in~\cite{Neumann_2007}]
\label{lem:neumann_extreme_points}
\gsemo constructs a population that contains a spanning tree corresponding to each of the extreme objective vectors $q_1$ and $q_r$ in expected time $O(m^2 |P|_{\max} (\log n + \log w_{max}))$.
\end{lemma}

Throughout this paper, we assume that the two edge weight
functions $w_1 \colon E \rightarrow W$ and $w_2 \colon E \rightarrow
W$ take on values from a small set of weight values $W \subset \mathbb{R}_{>0}$.
In particular, we are studying the case where $W$ only contains $2$ or $3$ different weights. Note that the settings that we consider in this paper still allow for a Pareto front that has a linear number of objective vectors with respect to the number of nodes $n$.

\section{General Structural Properties}
\label{sec3}

We now investigate some structural properties of Pareto-optimal solutions and assume that each of the two weight functions maps to a small set of weights $W \subset \mathbb{R}_{>0}$. 
\begin{lemma} \label{lem:fix_edges}
Assume that each edge weight is chosen from $\mathbb{R}_{>0}^2$. Then, the values $x_{w_{\min}}(T)$ and $x_{w_{\max}}(T)$ are identical for all Pareto-optimal spanning trees $T$.
\end{lemma}

\begin{proof}
%
\textbf{Edges with weight \boldsymbol{$(w_{\min}, w_{\min})$}:}
Suppose there exist two Pareto-optimal spanning trees $T$ and $T'$ such that
$x_{w_{\min}}(T) < x_{w_{\min}}(T')$, and let
$k := |T \setminus T'|$. Then there exists a sequence of $k$ edge exchanges that transforms $T$ into $T'$. In particular, at least one of these exchanges must increase the number of edges of type $(w_{\min}, w_{\min})$ with respect to $T$.

The weight $(w_{\min}, w_{\min})$ dominates every other weight $(a,b) \neq (w_{\min}, w_{\min})$, since
$w_{\min} \leq a$  and  $w_{\min} \leq b$,
and at least one of these inequalities is strict whenever $(a,b) \neq (w_{\min}, w_{\min})$.

Therefore, replacing an edge of weight $(a,b) \neq (w_{\min}, w_{\min})$ by an edge of weight $(w_{\min}, w_{\min})$ strictly improves at least one objective while not worsening the other. Hence, the resulting spanning tree dominates $T$, contradicting the Pareto optimality of $T$.
It follows that all Pareto-optimal spanning trees contain the same number of edges of type $(w_{\min}, w_{\min})$. 

\textbf{Edges with weight \boldsymbol{$(w_{\max}, w_{\max})$}:}
Suppose there exist two Pareto-optimal spanning trees $T$ and $T'$ such that
$x_{w_{\max}}(T) > x_{w_{\max}}(T')$, and
let $k := |T \setminus T'|$. Then there exists a sequence of $k$ edge exchanges that transforms $T$ into $T'$. In particular, at least one of these exchanges must decrease the number of edges of type $(w_{\max}, w_{\max})$ with respect to $T$.

Every weight $(a,b) \neq (w_{\max}, w_{\max})$ dominates $(w_{\max}, w_{\max})$, since
$a \leq w_{\max}$ and $b \leq w_{\max}$,
and at least one of these inequalities is strict whenever $(a,b) \neq (w_{\max}, w_{\max})$.

Therefore, replacing an edge of weight $(w_{\max}, w_{\max})$ by an edge of weight $(a,b) \neq (w_{\max}, w_{\max})$ strictly improves at least one objective while not worsening the other. Hence, the resulting spanning tree dominates $T$, contradicting the Pareto-optimality of $T$.
It follows that all Pareto-optimal spanning trees contain exactly the
same number of $(w_{\max}, w_{\max})$ edges, and therefore $x_{w_{\max}}(T)$ is identical for all Pareto-optimal trees. \qed
\end{proof}

\section{Two Distinct Weight Values}
\label{sec4}

In this section, we study the performance of \gsemo on the bi-MST
problem where each edge weight can take one of two values $\{a,b\}
\subseteq \mathbb{R}_{>0}$, with $a < b$, for each objective. 
We first analyze the structure of the Pareto front and then the runtime of \gsemo in this setting.

\begin{lemma} \label{lem:line_12}
For the bi-MST problem with edge weights from $\{a,b\}$, the $\mathrm{conv}(\mathcal{F})$ is a line segment. 
\end{lemma}
\begin{proof}
Let $T$ be a Pareto-optimal spanning tree.
Since $T$ is a tree,
$x_{aa} + x_{ab} + x_{ba} + x_{bb} = n-1$.
By Lemma~\ref{lem:fix_edges}, the number of edges of type $(a,a)$ and
$(b,b)$ is fixed for all Pareto-optimal spanning trees. Hence $x_{aa}$
and $x_{bb}$ are constants, therefore there exists a value $k$ (fixed for the Pareto front) such that $x_{ab} + x_{ba} = k$.
The fitness values of $T$ are
\[
f_1(T) = ax_{aa} + ax_{ab} + bx_{ba} + bx_{bb}, 
\qquad
f_2(T) = ax_{aa} + bx_{ab} + ax_{ba} + bx_{bb}.
\]
Since $x_{ba}=k-x_{ab}$, we obtain\\
\[
f_1(T) =  ax_{aa} + bx_{bb} + bk - (b-a)x_{ab},
\qquad
f_2(T) = ax_{aa} + bx_{bb} + ak + (b-a)x_{ab},
\]
therefore 
$f_1(T) + f_2(T) = 2ax_{aa} + 2bx_{bb} + (a+b)k$,
since $x_{aa}$, $x_{bb}$ and $k$ are constant over all Pareto-optimal trees, the sum $f_1(T) + f_2(T)$ is constant. So all Pareto points lie on a line of slope $-1$.\qed
\end{proof}

\subsection{Analysis of \gsemo}

We now analyze the behavior of \gsemo in the two-weight setting and derive bounds on its population size and runtime.

\begin{lemma}
\label{lem:2w}
  For the bi-MST problem in which each edge weight is taken from $\{a,b\}$, the population size of \gsemo after it has obtained a spanning tree for the first time is at most $n-1$.
\end{lemma}

\begin{proof}
We assume that a spanning tree has already been obtained which implies
that the whole population consists only of spanning trees, as an
arbitrary spanning tree strictly dominates all solutions that are not spanning trees. Each edge can only take weights in $\{a,b\}$. Consider a solution that has $k$ edges of weight $a$ according to the first weight function $w_1$ and therefore $(n-1)-k$ edges of weight $b$. The total weight according to $w_1$ is therefore $k\cdot a + (n-1-k) \cdot b$. Considering all possible spanning trees, $w_1$ can only take on values in $W^*=\{k\cdot a + (n-1-k) \cdot b \mid 0 \leq k \leq n-1\}$. We have $|W^*| \leq n-1$ and as \gsemo maintains for each fixed value in $W^*$ at most one solution due to dominance, the maximum population size of \gsemo once it has obtained a spanning tree is at most $n-1$. \qed
\end{proof}
\begin{theorem}
For the bi-MST problem in which each edge weight is taken from $\{a,b\}$ with $a<b$, the expected time of \gsemo to obtain the entire Pareto front is $O(m^2 n \log n)$.
\end{theorem}
\begin{proof}
By Lemma~\ref{lem:neumann}, \gsemo reaches a population $P$ consisting only of spanning trees in expected time $O(m \log n)$.

According to Lemma~\ref{lem:2w}, the maximum population size of \gsemo once it has obtained a spanning tree is at most $n-1$.
Hence, by Lemma~\ref{lem:neumann_extreme_points}, \gsemo obtains spanning trees corresponding to the extreme Pareto-optimal objective vectors $q_1$ and $q_r$ in expected time $O(m^2 n \log n)$.

Let the Pareto front consist of the objective vectors $q_1, q_2, \ldots, q_r$ with $r \leq n$, ordered by increasing $f_1$-value. Suppose that the current population of \gsemo already contains spanning trees corresponding to $q_1, \dots, q_i$, for $1 \leq i < r$. To obtain a spanning tree corresponding to $q_{i+1}$, Lemma~\ref{lem:line_12} implies that one edge of type $(a,b)$ must be exchanged with an edge of type $(b,a)$. There are $r-i$ such exchange operations that can be applied to the solution of $q_i$ in order to obtain a solution with objective vector $q_{i+1}$.

The probability of performing one of these $r-i$ exchange operations is at least $\frac{1}{e} \cdot \frac{r-i}{m^2 n}$. Summing over the possible values of $i$, the expected time to obtain all Pareto-optimal objective vectors can be bounded by
\begin{align*}
e m^2 n \sum_{i=1}^{r-1} \frac{1}{r-i}
= O\!\left(m^2 n \sum_{i=1}^{r-1} \frac{1}{r-i}\right)
= O(m^2 n \log r)
= O(m^2 n \log n). \tag*{\qed} 
\end{align*}

\end{proof}

\subsection{Experimental Results}
\label{sec:experimental_12}
We now present experimental results to evaluate the performance of GSEMO on instances with two distinct edge weights.

All experiments were performed on a virtual server at the University of Lille, which has an Intel(R) Xeon(R) Silver 4114 @ 2.20 GHz processor and 190 GB of RAM.

Table~\ref{tab:12} reports the experimental results for $a=1$ and $b=2$ across different values of $n$ and $p$. Here, $n$ denotes the number of vertices of the graph. 
The parameter $p$ represents the percentage of edges assigned mixed weights $(1,2)$ and $(2,1)$, while the remaining edges have weights $(1,1)$ or $(2,2)$.
For each parameter setting, the results are averaged over 30 randomly generated complete graphs. 
We report three performance metrics. The metric \textit{size} denotes the cardinality of the Pareto-optimal set, \textit{iter} represents the number of iterations required to identify all Pareto-optimal solutions, and \textit{runtime} corresponds to the execution time in seconds.

The results show that the performance of \gsemo strongly depends on both $p$ and $n$. 
For small values of $p$, the Pareto-optimal set consists of a single objective vector. 
This is due to the existence of a spanning tree in which all edges have weight $(1,1)$, yielding a unique optimal solution for both objectives. 
As $p$ increases, edges with mixed weights $(1,2)$ and $(2,1)$ become more frequent, introducing a trade-off between the objectives and leading to a larger Pareto-optimal set.
\begin{table}[t]
\centering
\caption{Performance of \gsemo with edge weights $\{1,2\}$. Iteration values are scaled by $10^6$, and runtime values are given in seconds. Values in parentheses denote standard deviations.}
\label{tab:12}
\resizebox{\textwidth}{!}{%
\begin{tabular}{
cc|c
>{\columncolor{gray!10}}c 
c 
>{\columncolor{gray!10}}c 
c 
>{\columncolor{gray!10}}c 
c 
>{\columncolor{gray!10}}c 
}

\toprule
$n$ & metric & \multicolumn{8}{c}{$p$} \\
\cmidrule(lr){3-10}
 &  & 50 & 60 & 70 & 80 & 90 & 94 & 98 & 100 \\
\midrule

\multirow{3}{*}{20}
& size
& $1.13\ [0.34]$ & $1.33\ [0.54]$ & $2.03\ [0.87]$ & $4.17\ [1.95]$
& $11.17\ [1.75]$ & $14.63\ [1.72]$ & $18.10\ [0.91]$ & $20.00\ [0.00]$ \\
& iter
& $0.086\ [0.049]$ & $0.174\ [0.133]$ & $0.186\ [0.116]$ & $0.260\ [0.113]$
& $0.430\ [0.338]$ & $0.470\ [0.374]$ & $0.530\ [0.306]$ & $0.529\ [0.346]$ \\
& runtime
& $0.40\ [0.23]$ & $0.80\ [0.60]$ & $0.86\ [0.51]$ & $1.21\ [0.52]$
& $2.00\ [1.49]$ & $12.02\ [1.60]$ & $2.31\ [1.34]$ & $2.55\ [1.69]$ \\

\midrule

\multirow{3}{*}{50}
& size
& $1.00\ [0.00]$ & $1.00\ [0.00]$ & $1.03\ [0.18]$ & $1.23\ [0.50]$
& $5.60\ [1.87]$ & $14.73\ [3.21]$ & $37.47\ [2.31]$ & $50.00\ [0.00]$ \\
& iter
& $1.40\ [0.60]$ & $2.02\ [0.99]$ & $3.27\ [1.71]$ & $7.03\ [3.08]$
& $14.67\ [4.22]$ & $20.00\ [6.89]$ & $21.53\ [6.95]$ & $21.68\ [9.57]$ \\
& runtime
& $39.45\ [16.76]$ & $55.93\ [27.77]$ & $87.24\ [42.47]$ & $181.61\ [78.37]$
& $383.82\ [121.91]$ & $531.10\ [178.14]$ & $580.36\ [207.96]$ & $627.71\ [282.79]$ \\

\midrule

\multirow{3}{*}{100}
& size
& $1.00\ [0.00]$ & $1.00\ [0.00]$ & $1.00\ [0.00]$ & $1.00\ [0.00]$
& $1.47\ [0.67]$ & $6.37\ [2.39]$ & $51.50\ [4.06]$ & $100.00\ [0.00]$ \\
& iter
& $13.52\ [6.04]$ & $19.28\ [5.34]$ & $30.61\ [12.03]$ & $64.07\ [30.62]$
& $182.19\ [57.62]$ & $300.09\ [64.58]$ & $433.48\ [133.82]$ & $568.51\ [184.62]$ \\
& runtime
& $1,366\ [628]$ & $1,971\ [574]$ & $3,164\ [1,335]$ & $6,459\ [3,023]$
& $18,614\ [6,060]$ & $28,602\ [5,901]$ & $44,613\ [12,849]$ & $58,483\ [19,334]$ \\

\bottomrule
\end{tabular}
}
\end{table}

\section{Three Distinct Weight Values: Equidistant Case}
\label{sec5}
In this section, we study the performance of \gsemo on the bi-MST problem in which each edge weight takes values from $\{a,b,c\} \subset \mathbb{R}_{>0}$ with $a < b < c$ and $b-a = c-b = \delta > 0$. 
We analyze both the structure of the Pareto front and the runtime of \gsemo in this setting.

Let $T$ be a Pareto-optimal spanning tree. If there existed a feasible exchange replacing an edge of type $(u,v)$ in $T$ by an edge of type $(u',v')$
where $(u',v') \dominates (u,v)$, then the resulting spanning tree $T'$ would
dominate $T$, contradicting the Pareto optimality of $T$. Hence any exchange between Pareto-optimal solutions must occur between non-dominated edge types and must represent a trade-off between the objectives $f_1$ and $f_2$.
An exchange of an edge of type $(u,v)$ with an edge of type $(u',v')$ is denoted by $(u,v) \leftrightarrow (u',v')$.
Such an exchange induces a change $(u'-u,\, v'-v)$ in the objective space.
Under the considered instance structure, the admissible Pareto-preserving exchanges
induce the following changes in the objective values:
\[
\begin{aligned}
(a,b) &\leftrightarrow (b,a) &&\rightarrow (\delta,-\delta), &&
(a,b) &\leftrightarrow (c,a) &&\rightarrow (2\delta,-\delta),\\
(a,c) &\leftrightarrow (b,a) &&\rightarrow (\delta,-2\delta), &&
(a,c) &\leftrightarrow (c,a) &&\rightarrow (2\delta,-2\delta).
\end{aligned}
\]
Thus, every Pareto-preserving exchange yields a change vector of the form
$(\delta,-\delta)$, $(\delta,-2\delta)$, $(2\delta,-\delta)$, $(2\delta,-2\delta)$.
In particular, the only distinct slopes that can occur are
$\frac{\Delta f_2}{\Delta f_1}\in\left\{-2,-1,-\tfrac12\right\}$.
\begin{lemma}\label{lem:breakpoints_123_eq}
For the equidistant three weight bi-MST problem, the efficient boundary of $\mathrm{conv}(\mathcal{F})$ has at most two interior breakpoints. 
\end{lemma}

\begin{proof}
$\mathrm{conv}(\mathcal{F})$ is a convex polygon in $\mathbb{R}^2$. Each of its line segments has a slope $s=\Delta f_2/\Delta f_1$ with $\Delta f_1>0$ and $\Delta f_2<0$.
By assumption, any Pareto-preserving move has slope in the set
$\{-2,-1,-\tfrac12\}$, so only these three slopes can appear on the boundary.

A standard property of planar convex polygons is that the slopes of consecutive
edges along the lower hull are monotonically increasing when traversed from left
to right. Hence, the boundary can switch between the three possible slopes at
most twice. Therefore, it consists of at most three line segments and has at most
two interior breakpoints. \qed
\end{proof}
\begin{lemma}
\label{lem:breakpoints_123}
For the equidistant three weight bi-MST problem, all Pareto-optimal objective vectors lie on the efficient boundary of $\mathrm{conv}(F)$.
\end{lemma}

\begin{proof}
Let $q_1,q_2,\dots,q_r$ denote the Pareto-optimal objective vectors, ordered by increasing $f_1$.
By Lemma~\ref{lem:breakpoints_123_eq}, the efficient boundary of $\mathrm{conv}(F)$ has at most two breakpoints, denoted by $q_a$ and $q_b$. Every Pareto-preserving exchange induces one of the difference vectors
$(\delta,-2\delta),\ (\delta,-\delta),\ (2\delta,-2\delta),\ (2\delta,-\delta)$,
whose slopes are $-2$, $-1$, and $-\frac12$. Along the efficient boundary, the slopes appear in nondecreasing order, yielding at most three linear phases. Phase~1 is a segment of slope $-2$ connecting $q_1$ to $q_a$, Phase~2 is a segment of slope $-1$ connecting $q_a$ to $q_b$, and Phase~3 is a segment of slope $-\frac{1}{2}$ connecting $q_b$ to $q_r$.

\textbf{Phase 1 (From $q_1$ to $q_a$ (slope $-2$)):}
In this phase, consecutive Pareto points satisfy
$q_{i+1}-q_i=(\delta,-2\delta)$.
Hence, all points between $q_1$ and $q_a$ lie on an affine line of slope $-2$.
Assume for contradiction that there exists a Pareto-optimal point $p$ in this region that does not lie on this line. Then $p$ must arise from one of the other exchange vectors, whose slope is strictly larger than $-2$. Therefore, $p$ lies strictly above the slope $-2$ line through $q_1$, and is dominated by points on that line, contradicting Pareto optimality. Hence, all Pareto-optimal points between $q_1$ and $q_a$ lie on the efficient boundary.

\textbf{Phase 2 (From $q_a$ to $q_b$ (slope $-1$)):}
At $q_a$ the slope changes from $-2$ to $-1$. In this region,
$q_{i+1}-q_i\in\{(\delta,-\delta),(2\delta,-2\delta)\}$.
Both vectors lie on the same line of slope $-1$. Hence, all Pareto-optimal points
in this region lie on a common affine line.

Suppose there exists a Pareto-optimal point $p$ not on this line. Such a point would have to arise from an exchange vector with slope either $-2$ or $-\frac12$.
If the slope were $-2$, the slope along the efficient boundary would decrease, contradicting convexity. If the slope were $-\frac12$, then $p$ would lie strictly above the slope $-1$ segment between $q_a$ and $q_b$, and hence be dominated by a convex combination of $q_a$ and $q_b$. This is again impossible. Thus, all Pareto-optimal points in this phase lie on the efficient boundary.

\textbf{Phase 3 (From $q_b$ to $q_r$ (slope $-\frac12$)):}
In this final phase, consecutive Pareto points satisfy
$q_{i+1} - q_i=(2\delta,-\delta)$,
and therefore all points lie on a line of slope $-\frac12$.

Suppose, for contradiction, that there exists a Pareto-optimal point $p$ in this region that does not lie on this line. Then $p$ must be generated by an exchange vector with slope $-2$ or $-1$, both of which are strictly smaller than $-\frac12$. However, along the efficient boundary of a convex polygon, slopes are nondecreasing from left to right. Since the slope has already increased to $-\frac12$ at $q_b$, it cannot decrease again. This contradiction shows that all Pareto-optimal points between $q_b$ and $q_r$ lie on the efficient boundary.

Since in each phase all Pareto-optimal points lie on a single line segment of the convex hull, every Pareto-optimal objective vector lies on the efficient boundary of $\mathrm{conv}(F)$. \qed
\end{proof}

\subsection{Analysis of \gsemo}
We now analyze the behavior of \gsemo in the three-weight setting and derive bounds on its population size and runtime.
\begin{lemma}
\label{lem:3w}
  For the bi-MST problem in which each edge weight is taken from $\{a,b,c\}$ with  $a < b < c$ and $b-a = c-b = \delta > 0$, the population size of \gsemo after it has obtained a spanning tree for the first time is at most $2n-1$.
\end{lemma}
\begin{proof}
Consider weight function $w_1$. Assuming that a spanning tree has  already been obtained, an accepted solution can only take on values in $W^*=\{(n-1)a, (n-1)a+\delta, \ldots, (n-1)a+2(n-1)\delta\}$ with respect to $w_1$. We have $|W^*| \leq 2(n-1)+1=2n-1$. Hence, the population size of \gsemo after it has obtained a spanning tree for the first time is at most $2n-1$.
\qed
\end{proof}

\begin{theorem}
For the bi-MST problem in which each edge weight is taken from $\{a,b,c\}$ with  $a < b < c$ and $b-a = c-b = \delta > 0$, the expected time of \gsemo to obtain the entire Pareto front is $O(m^2 n \log n)$.
\end{theorem}
\begin{proof}
By Lemma~\ref{lem:neumann}, \gsemo reaches a population consisting only of spanning trees in expected time $O(m \log n)$.
According to Lemma~\ref{lem:3w}, the maximum population size of \gsemo once it has obtained a spanning tree for the first time is $2n-1$.
Consequently, by Lemma~\ref{lem:neumann_extreme_points}, \gsemo obtains spanning trees corresponding to the extreme Pareto-optimal objective vectors $q_1$ and $q_r$ in expected time $O(m^2 n \log n)$.

By Lemma~\ref{lem:breakpoints_123}, the Pareto front consists of three linear phases corresponding to slopes $-2$, $-1$, and $-\tfrac{1}{2}$.
We analyze the time required to traverse one such phase.
Let the Pareto front within a phase consist of the objective vectors $q_u, q_{u+1}, \dots, q_v$ with $v \leq n$, ordered by increasing $f_1$. Let $d = v - u$ denote the length of this phase.

Suppose that the current population of \gsemo already contains spanning trees corresponding to $q_u, \dots, q_i$ for some $u \leq i < v$. To obtain a spanning tree corresponding to $q_{i+1}$, Lemma~\ref{lem:breakpoints_123} implies that only one edge exchange is needed. Moreover, there are $d-i$ such exchange operations that can be applied to a solution corresponding to $q_i$ in order to obtain a solution with objective vector $q_{i+1}$.
The probability of performing one of these $d-i$ exchange operations in a single iteration is at least
$\frac{1}{e} \cdot \frac{d-i}{m^2 (2n-1)}.$

Summing over all values of $i$, the expected time to obtain all Pareto-optimal objective vectors in the phase can be upper bounded by
$$e m^2 (2n-1)\sum_{i=1}^{d-1} \frac{1}{d-i}
= O(m^2 n) \sum_{i=1}^{d-1} \frac{1}{d-i}
= O(m^2 n \log d)
= O(m^2 n \log n).$$
Thus, the expected time to obtain Pareto-optimal objective vectors in one phase is $O(m^2 n \log n)$. Since the Pareto front consists of at most three phases, the total expected time to obtain all Pareto-optimal objective vectors is $O(m^2 n \log n)$. \qed
\end{proof}

\subsection{Experimental Results}
We now present experimental results to evaluate the performance of \gsemo under different instance parameters.

Table~\ref{tab:123} presents the results for the weight set $\{1,2,3\}$.
The parameters $n$ and the performance metrics are defined as in Subsection~\ref{sec:experimental_12}. The parameter $p$ denotes the percentage of edges assigned mixed weights $(1,2)$, $(1,3)$, $(2,1)$, $(2,2)$, $(2,3)$, $(3,1)$, and $(3,2)$, while the remaining edges have weights $(1,1)$ or $(3,3)$. For each parameter setting, the values are averaged over 30 randomly generated complete graphs.

The observed behavior follows a similar trend to the $\{1,2\}$ case, but with a noticeably higher complexity. For small values of $p$, the Pareto-optimal set remains small, often consisting of only a few solutions, and the algorithm converges relatively quickly. 
As $p$ increases, the number of mixed edge types grows, leading to a richer set of trade-offs between the objectives and consequently a larger Pareto front.
This increase in the size of the Pareto-optimal set results in a substantial growth in both the number of iterations and the runtime. 

\begin{table}[t]
\centering
\small
\caption{Performance of \gsemo with edge weights $\{1,2,3\}$. Iteration values are scaled by $10^6$, and runtime values are given in seconds. Values in parentheses denote standard deviations.}
\label{tab:123}
\resizebox{\textwidth}{!}{%
\begin{tabular}{
cc|c
>{\columncolor{gray!10}}c
c
>{\columncolor{gray!10}}c
c
>{\columncolor{gray!10}}c
c
>{\columncolor{gray!10}}c
}
\toprule
$n$ & metric & \multicolumn{8}{c}{$p$} \\
\cmidrule(lr){3-10}
 &  & 50 & 60 & 70 & 80 & 90 & 94 & 98 & 100 \\
\midrule

\multirow{3}{*}{20}
& size
& $1.17\ [0.45]$ & $1.37\ [0.60]$ & $2.03\ [0.98]$ & $4.17\ [1.95]$
& $10.83\ [1.90]$ & $14.63\ [1.72]$ & $18.07\ [0.89]$ & $19.93\ [0.25]$ \\
& iter
& $0.104\ [0.073]$ & $0.166\ [0.098]$ & $0.201\ [0.121]$ & $0.467\ [0.225]$
& $1.039\ [0.623]$ & $1.611\ [0.666]$ & $1.729\ [0.797]$ & $2.675\ [1.425]$ \\
& runtime
& $0.46\ [0.31]$ & $0.74\ [0.44]$ & $0.90\ [0.53]$ & $2.08\ [0.99]$
& $4.74\ [2.91]$ & $7.56\ [3.03]$ & $8.26\ [3.92]$ & $12.46\ [6.65]$ \\

\midrule

\multirow{3}{*}{50}
& size
& $1.00\ [0.00]$ & $1.00\ [0.00]$ & $1.03\ [0.18]$ & $1.23\ [0.50]$
& $5.60\ [1.87]$ & $14.73\ [3.21]$ & $37.47\ [2.31]$ & $50.00\ [0.00]$ \\
& iter
& $1.20\ [0.60]$ & $1.72\ [0.71]$ & $3.74\ [1.42]$ & $7.42\ [3.17]$
& $23.24\ [9.00]$ & $41.76\ [17.15]$ & $66.96\ [29.70]$ & $83.12\ [37.55]$ \\
& runtime
& $31.88\ [15.11]$ & $45.89\ [20.00]$ & $103.08\ [39.06]$ & $203.74\ [84.92]$
& $660.82\ [255]$ & $1,019\ [446]$ & $1,694\ [757]$ & $2,428\ [1,120]$ \\

\midrule

\multirow{3}{*}{100}
& size
& $1.00\ [0.00]$ & $1.00\ [0.00]$ & $1.00\ [0.00]$ & $1.00\ [0.00]$
& $1.47\ [0.67]$ & $6.37\ [2.39]$ & $51.62\ [4.08]$ & $100\ [0.00]$ \\
& iter
& $11.59\ [5.41]$ & $17.05\ [7.27]$ & $26.91\ [8.61]$ & $61.63\ [23.81]$
& $217.52\ [79.67]$ & $419.62\ [138.80]$ & $926.44\ [242.02]$ & $1278.50\ [431.07]$ \\
& runtime
& $1,117\ [529]$ & $1,639\ [688]$ & $2,641\ [913]$ & $5,880\ [2,191]$
& $21,307\ [7,564]$ & $46,792\ [15,001]$ & $88,281\ [22,116]$ & $132,040\ [44,974]$ \\

\bottomrule
\end{tabular}
}
\end{table}

\section{Three Distinct Weight Values: Non-equidistant Case}
\label{sec6}
We now extend our analysis to the non-equidistant case with three distinct weight values, where each edge weight takes values from $\{a,b,c\} \subset \mathbb{R}_{>0}$ with $a<b<c$ and $b-a \neq c-b$.

In contrast to the previous settings, we now show that when there are three distinct weight values that are not equidistant, not all Pareto-optimal objective vectors lie on $\mathrm{conv}(\mathcal{F})$ anymore. By way of counterexample, we show in Example~\ref{emp:124} the existence of a graph for which a particular objective vector ($q_3$) is Pareto optimal, but is not on the convex hull.

\begin{example}
\label{emp:124}
Consider the graph $G=(V,E)$ with $V=\{a,b,c,d,e\}$ and bi-objective weights shown below.
The corresponding Pareto-optimal spanning trees are:
\[
\begin{aligned}
T_1 &= \{ab,ca,cd,de\}, & q_1 &= (7,11), & T_3 &= \{ab,bc,cd,ce\}, & q_3 &= (11,7), \\
T_2 &= \{bc,ca,cd,ce\}, & q_2 &= (9,8), & T_4 &= \{ac,bc,cd,ce\}, & q_4 &= (12,5).\\
\end{aligned}
\]

\begin{center}
\resizebox{0.65\textwidth}{!}{%
\begin{tikzpicture}[font=\small]
  \begin{scope}
    \coordinate (a) at (0,2.2);
    \coordinate (b) at (1.5,4.4);
    \coordinate (c) at (3,2.2);
    \coordinate (e) at (6,2.2);
    \coordinate (d) at (4.5,0);

    \draw (a) -- (b);
    \draw (b) -- (c);
    \draw (a) -- (c);
    \draw (c) -- (e);
    \draw (c) -- (d);
    \draw (d) -- (e);

    \foreach \v/\name in {a/a,b/b,c/c,d/d,e/e}{
      \draw[fill=white] (\v) circle (2pt);
    }

    \node[left]  at (a) {$a$};
    \node[above] at (b) {$b$};
    \node[above] at (c) {$c$};
    \node[right] at (e) {$e$};
    \node[below] at (d) {$d$};

    \path (a) -- (b) coordinate[midway] (mab);
    \path (b) -- (c) coordinate[midway] (mbc);
    \path (a) -- (c) coordinate[midway] (mac);
    \path (c) -- (e) coordinate[midway] (mce);
    \path (c) -- (d) coordinate[midway] (mcd);
    \path (d) -- (e) coordinate[midway] (mde);

    \node[left]  at ($(mab)+(-0.05,0.10)$) {$(1,4)$};
    \node[above] at ($(mbc)+(0.30,0.05)$) {$(4,1)$};
    \node[below] at ($(mac)+(0.00,-0.05)$) {$(2,2)$};
    \node[above] at ($(mce)+(0.00,0.05)$) {$(4,1)$};
    \node[left]  at ($(mcd)+(0.00,0.00)$) {$(2,1)$};
    \node[right] at ($(mde)+(0.15,0.00)$) {$(2,4)$};
  \end{scope}

  \begin{scope}[xshift=8.2cm, yshift=0cm]
    \draw[->, line width=0.9pt] (0,0) -- (4.0,0) node[right] {$f_2$};
    \draw[->, line width=0.9pt] (0,0) -- (0,4.0) node[above] {$f_1$};

    \draw[gray!60, very thin] (0,0) grid[xstep=0.9,ystep=0.9] (4.0,4.0);

    \coordinate (q1) at (7/4,11/4);
    \coordinate (q2) at (9/4,8/4);
    \coordinate (q3) at (11/4,7/4);
    \coordinate (q4) at (12/4,5/4);

    \draw[dashed, line width=0.9pt] (q1) -- (q2) -- (q4);

    \draw[fill=white] (q1) circle (2.2pt);
    \draw[fill=white] (q2) circle (2.2pt);
    \draw[fill=black] (q3) circle (2.2pt);
    \draw[fill=white] (q4) circle (2.2pt);

    \node[above right] at (q1) {$q_1$};
    \node[above right] at (q2) {$q_2$};
    \node[above right] at (q3) {$q_3$};
    \node[above right] at (q4) {$q_4$};
  \end{scope}
\end{tikzpicture}%
}
\end{center}
\end{example}

\subsection{Multiple double triangles}
We generalize the graph in Example~\ref{emp:124} by taking a sequence of $k$ copies of the double triangle graph therein.
We call the first triangle consisting of the edges (1,4), (4,1), and (2,2) a type $A$ triangle and the other one a type $B$ triangle, as illustrated in Figure~\ref{fig:k-triangles}.
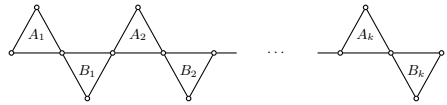
\begin{figure}[ht]
  \centering
    \begin{tikzpicture}[scale=0.6, transform shape, start chain=chain]
    \tikzstyle{vertex}=[draw,circle,inner sep=1pt]
      \foreach \i in {1,...,5}{
        \node[on chain,vertex] {};
      }
      \node[on chain] {$\cdots$};
      \foreach \i in {1,...,3}{
        \node[on chain,vertex] {};
      }
      \foreach \x [remember=\x as \lastx (initially 1)] in {2,...,5}{
        \tikzmath{
          integer \sign;
          \sign = (-1)^\x; 
        }
        \path[draw] (chain-\lastx) -- (chain-\x)
        node[pos=0.5,yshift=\sign cm,vertex] (c-\lastx-\x) {};
        \draw (chain-\lastx) -- (c-\lastx-\x);
        \draw (c-\lastx-\x) -- (chain-\x);
      }
      \foreach \x [remember=\x as \lastx (initially 7)] in {8,...,9}{
        \tikzmath{
          integer \sign;
          \sign = (-1)^\x; 
        }
        \path[draw] (chain-\lastx) -- (chain-\x)
        node[pos=0.5,yshift=\sign cm,vertex] (c-\lastx-\x) {};
        \draw (chain-\lastx) -- (c-\lastx-\x);
        \draw (c-\lastx-\x) -- (chain-\x);
      }
      \draw (chain-5) -- +(0.5,0);
      \draw (chain-7) -- +(-0.5,0);

      \node[below=3mm of c-1-2] {$A_1$};
      \node[above=3mm of c-2-3] {$B_1$};
      \node[below=3mm of c-3-4] {$A_2$};
      \node[above=3mm of c-4-5] {$B_2$};
      \node[below=3mm of c-7-8] {$A_k$};
      \node[above=3mm of c-8-9] {$B_k$};

    \end{tikzpicture}
    \caption{A sequence of $k$ double triangles.}
    \label{fig:k-triangles}
\end{figure}

\begin{lemma}
\label{lem:3k+1}
For any $k \geq 1$, the bi-MST instance consisting of $k$ pairs of triangles of type $A$ and type $B$ has exactly $3k+1$ Pareto-optimal objective vectors.
\end{lemma}
\begin{proof}
For a single triangle pair, Example~\ref{emp:124} shows that the only Pareto-optimal objective vectors are $q_1, q_2, q_3, q_4$. Since each triangle pair is independent of the others in the construction, a spanning tree of the whole graph with $k$ such pairs is obtained by choosing one local spanning tree configuration for each pair. The objective vector of the resulting spanning tree is therefore the sum of the objective contributions of the $k$ local spanning trees for each triangle pair. 

Moreover, in a Pareto-optimal global solution, each pair must contribute a non-dominated local vector. Hence, each pair contributes one of $q_1, q_2, q_3, q_4$. Consequently, every Pareto-optimal objective vector of the whole graph can be written as
\[
Q(\alpha,\beta,\gamma,\delta) =\alpha q_1+\beta q_2+\gamma q_3+\delta q_4, 
 \qquad
\alpha+\beta+\gamma+\delta=k,
\]
where $\alpha,\beta,\gamma,\delta \in \{0,1,\ldots,k\}$.
Now we identify which combinations can be Pareto-optimal.

\textbf{Step 1: At most one copy of \boldsymbol{$q_3$} can occur.}
Indeed,
$2q_3=(22,14)$
and
$q_2+q_4=(21,13)$,
so $q_2+q_4 \prec 2q_3$. Hence any vector containing at least two copies of $q_3$ is dominated by replacing two copies of $q_3$ with one copy of $q_2$ and one copy of $q_4$. Therefore, every Pareto-optimal vector must satisfy
$\gamma \leq 1$.

\textbf{Step 2: If \boldsymbol{$\gamma=1$}, then \boldsymbol{$\alpha=0$}.} 
Indeed,
$q_1+q_3=(18,18)$
and
$2q_2=(18,16)$,
so $2q_2 \prec q_1+q_3$. Thus, if a vector contains both $q_1$ and $q_3$, it is dominated. Hence, whenever $\gamma=1$, we must have $\alpha=0$.

\textbf{Step 3: If \boldsymbol{$\gamma=0$}, then \boldsymbol{$\alpha=0$} or \boldsymbol{$\delta=0$}.}
Indeed,
$q_1+q_4=(19,16)$
and
$2q_2=(18,16)$,
so $2q_2 \prec q_1+q_4$. Therefore, any vector containing both $q_1$ and $q_4$ is dominated. Hence, if $\gamma=0$, then either all summands come from $\{q_1,q_2\}$ or all summands come from $\{q_2,q_4\}$.

By the previous steps, every Pareto-optimal objective vector must belong to one of the following three families:
\[
\begin{aligned}
A_i &:= (k-i)q_1 + iq_2, && i=0,\dots,k, \\
B_j &:= (k-j)q_2 + jq_4, && j=1,\dots,k, \\
C_t &:= (k-1-t)q_2 + q_3 + tq_4, && t=0,\dots,k-1.
\end{aligned}
\]

Thus, there are at most $(k+1) + k + k = 3k+1$ candidate Pareto-optimal objective vectors.

\textbf{Step 4: All these vectors are pairwise non-dominated.}
We compute their explicit coordinates:
\[
\begin{aligned}
A_i &= (7k+2i,\; 11k-3i), && i=0,\dots,k, \\
B_j &= (9k+3j,\; 8k-3j), && j=1,\dots,k, \\
C_t &= (9k+2+3t,\; 8k-1-3t), && t=0,\dots,k-1.
\end{aligned}
\]

Now consider the sequence
$A_0,\ A_1, \dots,\ A_k,\ C_0,\ B_1,\ C_1,\ B_2,\ \dots,\ C_{k-1},\ B_k$.
Its consecutive differences for $i=0,\dots,k-1$ are
\[
\begin{aligned}
A_{i+1}-A_i &=(2,-3), && C_0-A_k =(2,-1), \\
B_{i+1}-C_i &=(1,-2), && C_{i+1}-B_{i+1} =(2,-1).
\end{aligned}
\]
Hence, along this sequence, the first objective strictly increases while the second strictly decreases. Therefore, no vector in the sequence dominates another, and all of them are Pareto-optimal.
Thus, the number of Pareto-optimal objective vectors is exactly $3k+1$. \qed
\end{proof}

\begin{theorem}
For any $k \geq 1$, the bi-MST instance consisting of $k$ pairs of triangles of type A and type B has exactly $k$ Pareto-optimal objective vectors that do not lie on $\mathrm{conv}(\mathcal{F})$.
\end{theorem}
\begin{proof}
By Lemma~\ref{lem:3k+1}, the Pareto-optimal objective vectors are exactly those in the three families
\[
\begin{aligned}
A_i &= (7k+2i,\; 11k-3i), && i=0,\dots,k, \\
B_j &= (9k+3j,\; 8k-3j), && j=1,\dots,k, \\
C_t &= (9k+2+3t,\; 8k-1-3t), && t=0,\dots,k-1.
\end{aligned}
\]
The vectors $A_i$ lie on a line of slope $-\frac32$, since $A_{i+1}-A_i=(2,-3)$, and the vectors $B_j$ lie on a line of slope $-1$, since $B_{j+1}-B_j=(3,-3)$.
Moreover, $B_1-A_k=(3,-3)$, so the vectors $A_0,\dots,\ A_k,\ B_1,\dots,\ B_k$ form the boundary of the convex hull.

Each vector $B_j$ satisfies $f_1+f_2= 17k$, while each vector $C_t$ satisfies $f_1+f_2= 17k+1$. Furthermore, since $C_{t+1}-C_t=(3,-3)$,
the vectors $C_t$ lie on a line of slope $-1$, parallel to the line containing the vectors $B_j$. As this line lies strictly above the line containing the vectors $B_j$, none of the vectors $C_t$ lies on the boundary of the convex hull. Since there are exactly $k$ such vectors, this completes the proof. \qed
\end{proof}

\subsection{Analysis of \gsemo}
We now analyze the runtime of \gsemo on the constructed instance.
\begin{theorem}
For the bi-MST instance consisting of $k$ pairs of triangles of type A and type B with edge weights taken from $\{1,2,4\}$, the expected time of \gsemo to obtain the entire Pareto-optimal is $O(n^4)$.
\end{theorem}
\begin{proof}
By Lemma~\ref{lem:neumann}, \gsemo reaches a population consisting only of spanning trees in expected time $O(m \log n)$.

By Lemma~\ref{lem:3k+1}, the Pareto front contains exactly $3k+1$ objective vectors. Since $k=\Theta(n)$ in the considered construction, the population size is bounded by $O(n)$. 
Consequently, by Lemma~\ref{lem:neumann_extreme_points}, \gsemo obtains spanning trees corresponding to the extreme Pareto-optimal objective vectors $q_1$ and $q_r$ in expected time $O(m^2 n \log n)$.

Assume that the current population does not yet contain the entire Pareto front. Then there exists a Pareto-optimal objective vector not yet represented in the population. By the structural properties of the Pareto front, every missing Pareto-optimal objective vector can be obtained from an already present neighboring Pareto-optimal solution by a single feasible edge exchange, that is, by a 2-bit flip.
Since the population size is at most $O(n)$, the probability of selecting a specific parent is at least $\frac{1}{O(n)}$. The probability of performing the required 2-bit flip is at least $\frac{1}{O(m^2)}$.
Therefore, the expected waiting time to obtain one new Pareto-optimal objective vector is at most $O(nm^2)$.

As the Pareto front contains $O(n)$ objective vectors, the expected time to obtain all of them is at most
$O(n)\cdot O(nm^2)=O(n^2m^2)$.
Since $m=\Theta(n)$, the total expected time to obtain all Pareto-optimal objective vectors is $O(n^4)$. \qed
\end{proof}

\subsection{Experimental Results}
The results in Figure~\ref{fig:boxplots} are obtained by running \gsemo $30$ times independently, and the boxplots illustrate the distribution of the number of iterations and runtime in seconds over these runs. As $k$ increases, both the number of iterations and the runtime grow significantly, indicating a clear dependence of the algorithm’s performance on the problem size. In particular, the growth appears linear, which is consistent with the theoretical runtime bound. 
\begin{figure}[t]
    \centering
    \caption{\gsemo performance over $30$ runs. Iterations (left) and runtime (right) vs.\ $k$ pairs triangle A and B.}
\includegraphics[width=0.9\textwidth]{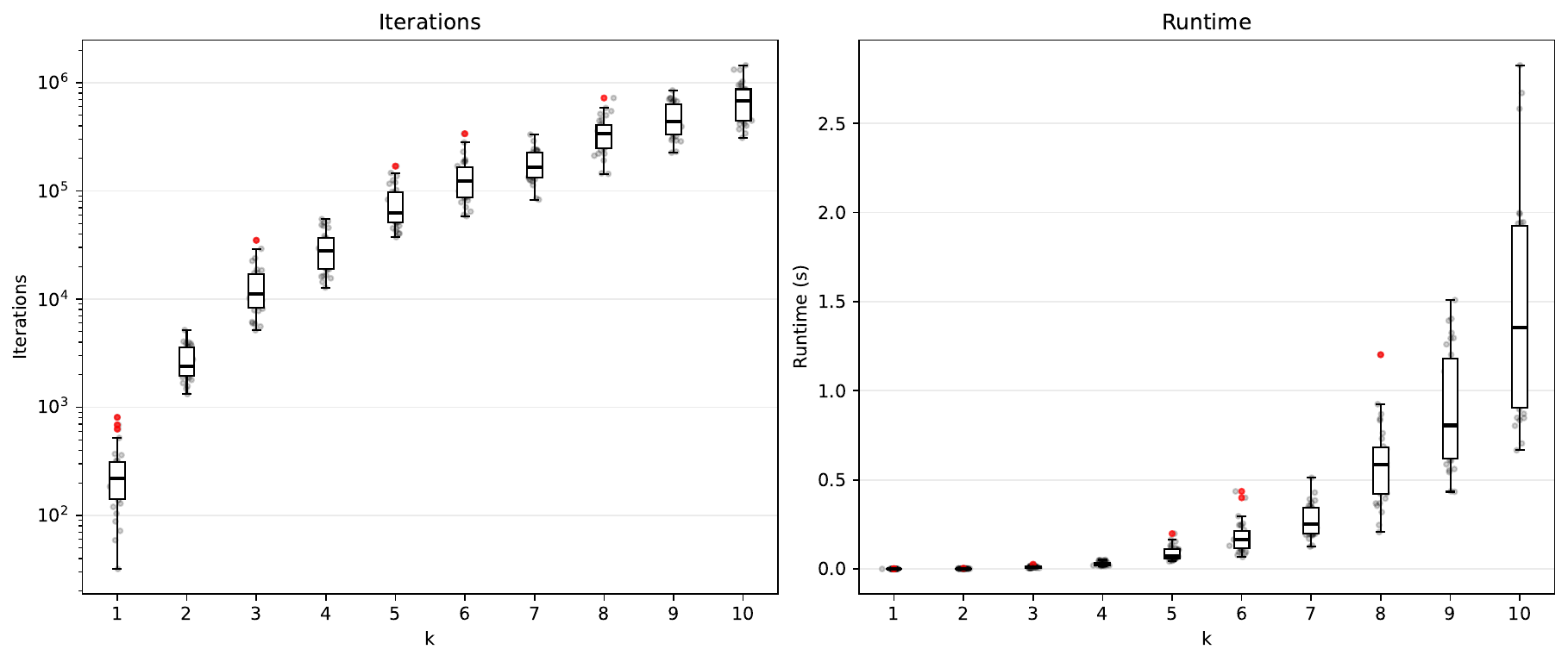}
    \label{fig:boxplots}
\end{figure}

\section{Translation of Results to Popular MOEAs}
\label{sec7}
While we investigated the setup for GSEMO, we note that the results
can be easily transferred to NSGA-II when using a sufficiently large
population size. This assertion is based on the observation that
NSGA-II never loses a Pareto-optimal objective vector if the
population size is at least four times the maximum number of
trade-offs~\cite{DBLP:journals/ai/ZhengD23}. The situation is
different for MOEA/D with linear weighting. We showed in
Section~\ref{sec6} that even when there are
only three distinct edge weight values, the Pareto front can already
be non-convex. This implies that some Pareto-optimal objective vectors are not included in an optimal population for the set up of MOEA/D studied in~\cite{DBLP:conf/nips/DoN0S23}.

\section{Conclusions}
\label{sec8}
In this paper, we studied the structural properties of the Pareto front and the runtime of \gsemo for the bi-MST problem under different edge-weight settings. We showed that the Pareto front has a simple linear structure in the two-weight case, while for three weights it exhibits a more complex geometry, including multiple linear phases and the fact that not all Pareto-optimal objective vectors lie on $\mathrm{conv}(\mathcal{F})$. For a constructed instance with $k$ pairs of triangles, we proved that the Pareto front contains exactly $3k+1$ objective vectors, of which exactly $k$ do not lie on $\mathrm{conv}(\mathcal{F})$. These findings highlight the strong impact of the problem structure on the performance of evolutionary multi-objective algorithms.

\begin{credits}
\subsubsection{\ackname} 
This work was supported by the Australian Research Council through grant FT200100536, the National Science Foundation through grant 2144080, the MOBLILEX scholarship, and the Hauts-de-France Region through funding for this PhD.
\end{credits}

\bibliographystyle{plain}
\bibliography{references}
\end{document}